\documentclass[twoside]{article}

\usepackage[accepted]{aistats2020}
\usepackage{amsmath}
\usepackage{amssymb}
\usepackage{amsthm}
\usepackage{amsfonts} 
\usepackage{natbib}
\usepackage{xcolor}
\usepackage{algorithm}
\usepackage{algorithmic}
\usepackage{soul}
\usepackage{graphicx}
\usepackage{hyperref}
\usepackage{bibunits}
\defaultbibliography{bibl}
\defaultbibliographystyle{apalike}

\newtheorem{theorem}{Theorem}

\newtheorem{proposition}[theorem]{Proposition}

\newcommand{\R}{\mathbb{R}}
\newcommand{\domain}{\mathbb{X}}

\newcommand{\expectation}{\mathbb{E}}
\newcommand{\E}{\mathbb{E}}
\newcommand{\uei}{\textrm{EI-UU}}
\newcommand{\argmax}{\mathop{\mathrm{argmax}}}

\newcommand{\reals}{\mathbb{R}}

\begin{document}

%
\runningtitle{Multi-attribute Bayesian optimization with interactive preference learning}

%

\twocolumn[

\aistatstitle{Multi-attribute Bayesian optimization\\ with interactive preference learning}

\aistatsauthor{ Raul Astudillo \And Peter I. Frazier }
\aistatsaddress{ Cornell University \And  Cornell University, Uber} 
]


\begin{abstract}
  We consider black-box global optimization of time-consuming-to-evaluate functions on behalf of a decision-maker (DM) whose preferences must be learned. Each feasible design is associated with a time-consuming-to-evaluate vector of attributes and each vector of attributes is assigned a utility by the DM's utility function, which may be learned approximately using preferences expressed over pairs of attribute vectors. Past work has used a point estimate of this utility function as if it were error-free within single-objective optimization. However, utility estimation errors may yield a poor suggested design. Furthermore, this approach produces a single suggested ``best'' design, whereas DMs often prefer to choose from a menu. We propose a novel {\it multi-attribute Bayesian optimization with preference learning} approach. Our approach acknowledges the uncertainty in preference estimation and implicitly chooses designs to evaluate that are good not just for a single estimated utility function but a range of likely ones. The outcome of our approach is a menu of designs and evaluated attributes from which the DM makes a final selection. We demonstrate the value and flexibility of our approach in a variety of experiments. 
\end{abstract}

\begin{bibunit}
\section{INTRODUCTION}

We begin with a motivating example: helping a cancer patient (the ``decision-maker'' or ``DM'') find the best treatment. Cancer treatments exhibit a range of efficacies, side effects and financial costs \citep{aning2012patient, wong2013cancer, marshall2016women}, referred to here as ``attributes''. Suppose a patient considers $k$ real-valued attributes when selecting a cancer treatment. Also suppose a time-consuming-to-evaluate black-box computational simulator can use the patient's medical history to compute the attributes, $f(x)\in\mathbb{R}^{k}$, of treatment $x$. The patient has an implicit preference over these attributes and our goal is to help her find her most preferred treatment by querying our simulator.

One existing approach, pursued within preference-based reinforcement learning \citep{wirth2017survey}, is to first learn a point estimate of the patient's preferences  \citep{dewancker2016,abbas2018foundations} and then optimize assuming this point estimate is correct.  We call this the ``point-estimate approach''.
This approach asks the patient for her preference between attribute vectors $f(x)$ and $f(x')$ corresponding to pairs of treatments $x$ and $x'$, and uses this information to learn a utility function $\widehat{U}:\R^k \rightarrow \R$, e.g., using preference learning with Gaussian processes \citep{chu2005preference}, such that the judgments are as consistent as possible with the estimated utility differences $\widehat{U}(f(x)) - \widehat{U}(f(x'))$.
It then solves $\max_x \widehat{U}(f(x))$ using a 
method for optimizing time-consuming-to-evaluate black-box functions, such as 
Bayesian optimization (BayesOpt) \citep{frazier2018tutorial}, assuming that the estimated utility function is correct. This approach, however, is not robust to residual uncertainty in preference estimates.
 

To better illustrate that becoming robust to uncertainty in preferences can improve performance, suppose that preference learning suggests that the patient's true utility function is close to one of $L$ possible functions $\{U_{\ell}\}_{\ell=1}^{L}$. Then, a better approach would be to offer the patient a set of treatments $\{x_{\ell}^*\}_{\ell=1}^{L}$, where $x_{\ell}^*\in \argmax_x U_\ell(f(x))$, and let her choose among them. 
This will provide near-optimal utility to the patient, while optimizing for a single point estimate of the utility function will not.
While this approach improves over the standard approach in the utility it provides, it requires solving $L$ optimization problems with a time-consuming-to-evaluate objective, which becomes computationally infeasible as $L$ grows. Our approach (described below) delivers similar utility gains using fewer queries to the objective function.

Another approach, which can be used when each attribute is a quantity that the patient wants to be as large (or small) as possible, is to use multi-objective Bayesian optimization \citep{knowles2006,abdolshah2019multi} to estimate the Pareto frontier. This approach, however, typically does not use interaction with the patient to focus optimization on the parts of the Pareto frontier most likely to contain the patient's preferred solution.  Intuitively, such information could accelerate optimization, especially when moderate or large numbers of attributes ($>3$) create high-dimensional Pareto frontiers and lead to many Pareto optimal solutions.

Motivated by the shortcomings of existing approaches, we propose \textit{optimization with preference learning}, which learns preferences from the DM's feedback and acknowledges uncertainty in these learned preferences.
In contrast with the point-estimate approach, 
our approach is significantly more robust to residual preference uncertainty because its optimization actions are appropriate for a range of plausible utility functions. 
In contrast with multi-objective optimization approaches, learned preferences allow our approach to use fewer objective function queries by focusing optimization on portions of the attribute space most likely to be preferred by the DM.
Our approach, therefore, fills an important gap between today's single-objective optimization approaches, which assume perfect knowledge of preferences, and multi-objective optimization approaches, which do not provide a principled way to accommodate partial preference information.

We develop optimization with preference learning within the specific context of Bayesian optimization. We use
pairwise judgments from the DM to form a Bayesian posterior distribution over her utility function and
model the attributes $f$ with a multi-output Gaussian process. We then use one of two novel acquisition functions, the {\it expected improvement under utility uncertainty} (EI-UU) or {\it Thompson sampling under utility uncertainty} (TS-UU), to iteratively choose designs $x$ at which to evaluate $f$. 
Optionally, during optimization, additional DM's judgments on the evaluated designs may be incorporated into our posterior distribution on the utility.
At the conclusion of optimization, a menu of designs is shown to the DM, who makes a final selection.

Our proposed acquisition functions, EI-UU and TS-UU, generalize existing Bayesian optimization acquisition functions to the optimization with preference learning setting.   
EI-UU is more challenging to maximize than its classical counterpart. However, we provide a simulation-based method for computing an unbiased estimator of its gradient,
which we use within 
a multi-start stochastic gradient optimization method.



    

The reminder of this paper is organized as follows. We first formalize our problem setting in \S\ref{sec:problem}, before defining the EI-UU acquisition function in \S\ref{sec:policies}, and reviewing other related work in \S\ref{sec:lit}. \S\ref{sec:experiments} presents numerical experiments, and \S\ref{sec:conclusion} concludes.

\section{PROBLEM SETTING}
\label{sec:problem}
We now formally describe our problem setting.

\subsection{Designs and Attributes}
We assume that both designs and attributes can be represented as vectors. More concretely, we assume that the space of designs can be represented as a compact set $\domain \subset \reals^{d}$, and attributes are given by a derivative-free time-consuming-to-evaluate black-box continuous function, $f:\domain \rightarrow \reals^{k}$. As is common in BayesOpt, we assume that $\domain$ is a simple set such as a hyperrectangle or a polytope, and that $d$ is not too large ($<20$).

\subsection{Decision-Maker's Preferences}
We assume that there is a DM whose preference over designs is characterized by the the designs' attributes through a \textit{Von Neumann-Morgenstern utility function} \citep{vonNeuman}, $U: \reals^{k} \rightarrow \reals$. This implies that the DM  (strictly) prefers a design $x$ over $x'$ if and only if $U(f(x)) > U(f(x'))$. Thus, of all the designs, the DM most prefers one in the set $\argmax_{x\in \domain}U(f(x))$.  
As is standard in preference learning \citep{furnkranz2010preference}, we assume that the DM can provide ordinal preferences between two designs $x$ and $x'$ when shown previously evaluated attribute vectors $f(x)$ and $f(x')$. 



\subsection{Interaction With the Decision-Maker and Computational Model}

In our approach, an algorithm interacts sequentially with a human DM and a time-consuming-to-evaluate objective function (typically a computer model). The algorithm interacts with the computational model simply by selecting a design $x$ and evaluating $f(x)$. We let $x_n$ indicate the $n^\text{th}$ point at which we evaluate $f$.  As is standard in BayesOpt, the first set of evaluations of $f$ is chosen uniformly at random or according to a space-filling design  over the feasible domain \citep{joseph2016space}, and subsequent evaluations are guided by an acquisition function described below in \S\ref{sec:policies}.

The algorithm interacts with the DM by receiving ordinal preferences between pairs of attribute vectors. We index interactions with the DM by $m$, letting $y_m$ and $y'_m$ refer to the attribute vectors queried in this interaction, and $a_m \in \{-1, 0,1\}$ indicating the DM's response, where $a_m = -1$ indicates a preference for $y'_m$, $a_m = 0$ indicates indifference, and $a_m = 1$ indicates preference for $y_m$. We let $m_n$ be the number of design pairs evaluated by the DM by the completion of the $n^\text{th}$ run of the computational model. We envision that the $y_m$ and $y'_m$ would typically be the attribute vectors for previously evaluated designs, $f(x_n)$ and $f(x_{n'})$, where $m \geq \max(m_n,m_{n'})$.

For concreteness, our numerical experiments assume that, before each evaluation of $f$, the DM provides feedback on one pair of designs chosen uniformly at random from among those previously evaluated.  Our framework easily supports other patterns of interaction. For example, it supports a setting where the DM provides feedback in a single batch after the first-stage evaluations of the computational model are complete,
either over random previously evaluated attribute vectors or using a more sophisticated and query-efficient selection of attribute vectors (see, e.g., \citealt{lepird2015bayesian}). It also supports a setting in which the DM provides feedback at a random series of time points on pairs of previously evaluated attribute vectors of her choosing.

\subsection{Statistical Model Over $f$}
As is standard in BayesOpt, we place a (multi-output) Gaussian process (GP) prior on $f$ \citep{alvarez2012kernels}, $ \mathcal{GP}(\mu,K)$, characterized by a  mean function, $\mu:\domain\rightarrow\R^{k}$, and a positive definite covariance function, $K:\domain\times\domain\rightarrow\mathbb{S}_{++}^{k}$\footnote{$\mathbb{S}_{++}^{k}$ denotes the cone of $k\times k$ positive definite matrices.}. Thus, after observing $n$ noise-free evaluations of $f$ at points $x_1,\dots , x_n$, the estimates of the designs' attributes are given by the posterior distribution on $f$, which is again a multi-output GP, $\mathcal{GP}\left(\mu_n,K_n\right)$, where $\mu_n$ and $K_n$ can be computed in closed form in terms of $\mu$ and $K$ \citep{liu2018remarks}.

\subsection{Statistical Model Over $U$}
We use Bayesian preference learning \citep{chu2005preference,lepird2015bayesian}
to infer a posterior probability distribution over the utility function, $U$, given preferences expressed by the DM.
Although this method is standard in the literature, we describe it here for completeness. 

We use a parametric family of utility functions  $\{U(\cdot; \theta) : \theta \in \Theta\}$,
(following, for example, \citealt{akrour2014programming, wirth2016model});
a prior probability distribution over $\theta$, $p^\theta$;
and a likelihood function, $L$, giving the conditional probability $L(a_m; U(y_m;\theta) - U(y'_m;\theta))$
of the DM expressing preference $a_m$ 
in response to an offered pair of attribute vectors $y_m$, $y'_m$
with utility difference $U(y_m;\theta) - U(y'_m;\theta)$.
The posterior distribution over $\theta$ after feedback on $m$ pairwise comparisons,
written 
$p^\theta_m(\theta)$, 
is then given by Bayes' rule:
\begin{equation*}
p^\theta_m(\theta) 
\propto p^\theta(\theta) \prod_m L(a_m; U(y_m;\theta) - U(y'_m;\theta)).
\end{equation*}
In our approach, we rely only on the ability to sample from this posterior distribution.

The most widely used parametric family of utility functions is linear functions, $U(y;\theta) = \theta^\top y$ \citep{wirth2017survey}, 
with other examples including linear functions over kernel-based feature spaces
\citep{wirth2016model,kupcsik2018learning}
and deep neural networks \citep{christiano2017deep}.
Commonly used likelihood functions include probit and logit \citep{wirth2017survey}. In our numerical experiments, for simplicity, we assume fully accurate preference responses, i.e., $L(a; \Delta) = 1\{a = \textnormal{sign}(\Delta)\}$, with parameteric families and priors described below. Although we assume parametric utility functions, conceptually, our approach generalizes to handle
nonparametric Bayesian preference learning (see, e.g., \citealt{chu2005preference}). However, this poses additional computational challenges as our approach internally performs optimization given samples of the utility function, which can be slow for nonparametric models.

\subsection{Measure of Performance}
We suppose that, after $N$ evaluations of the computational model (and $m_N$ judgments on attribute vector pairs), the DM selects her most preferred design among all evaluated designs.  Thus, the utility generated, given $\theta$, is
\begin{equation}
\label{eq:value}
\max_{i=1,\ldots, N}U(f(x_i);\theta),
\end{equation}
and we wish to adaptively choose designs to evaluate, $x_1,\ldots,x_N$, to maximize the expected value of \eqref{eq:value},
\begin{equation}
\label{eq:expected_value}
\expectation\left[\max_{i=1,\ldots, N}U(f(x_i);\theta)\right],
\end{equation}
where the expectation is taken over the prior on $\theta$ and the randomness in $x_1,\ldots,x_N$ (induced by the random first stage of samples and randomness in the DM's responses).

The full BayesOpt with preference learning loop is summarized in Algorithm ~\ref{alg:loop}.

\begin{algorithm}[h]
\begin{algorithmic}[1]
\caption{BayesOpt with preference learning}
\label{alg:loop}
\REQUIRE{Prior over $\theta$; GP prior over $f$.}
\STATE{Evaluate a few designs uniformly at random.}
\FOR{$n = 0,\ldots,N-1$}
\STATE{Choose $y_{n+1}, y_{n+1}'$ uniformly at random between attributes of previously evaluated designs.}
\STATE{Observe preference information, $a_{n+1}$, and update posterior distribution on $\theta$.}
\STATE{$x_{n+1} \gets \arg\max_{x\in \domain} \uei_n(x)$.}
\STATE{Observe evaluation of $f$ at $x_{n+1}$, and update posterior GP distribution on $f$.}
\ENDFOR
\RETURN{Pareto front of $\{f(x_1),\ldots, f(x_N)\}$}.
\end{algorithmic}
\end{algorithm}



\section{ACQUISITION FUNCTIONS}
\label{sec:policies}

We propose two novel acquisition functions, the Expected Improvement under Utility Uncertainty (EI-UU), 
and Thompson Sampling under Utility Uncertainty (TS-UU), 
for selecting points at which to query $f$.
The bulk of our development and analysis focuses on EI-UU, since this is the more difficult of the two to optimize, and this acquisition function performs best in numerical experiments. The description of TS-UU is deferred to ~\ref{append:thompson}.


\subsection{Expected Improvement Under Utility Uncertainty (EI-UU)}
\label{sec:uei}
Expected improvement is arguably the most popular acquisition function in BayesOpt. It has been successfully generalized for multi-objective and constrained optimization \citep{emmerich2006single,gardner14}, and we next show that it can be naturally generalized to our setting  as well by extending expected improvement's one-step optimality analysis \citep{jones1998efficient,frazier2018tutorial}.

 After evaluating designs $x_1,\ldots,x_n$, the utility obtained by the DM when she selects her most preferred design among this set is
\begin{equation*}
  U_n^*(f;\theta) :=  \max_{i=1,\ldots,n}U(f(x_i);\theta).
\end{equation*}
On the other hand, if we evaluate one more design, $x$, the utility obtained by the DM increases by
\begin{align*}
  &\phantom{{}={}} \max\left\{ U(f(x);\theta), U_n^*(f;\theta)\right\} -  U_n^*(f;\theta)\\ 
  &= \left\{ U(f(x);\theta) - U_n^*(f;\theta)\right\}^{+}.
\end{align*}
This difference measures improvement from sampling $x$. Thus, a natural sampling policy is to evaluate the design that maximizes the expected improvement
\begin{equation}
   \uei_n(x):= \expectation_{n}\left[\left\{ U(f(x);\theta) - U_n^*(f;\theta)\right\}^{+} \right],
\end{equation}
where the expectation is over both $f(x)$ and $\theta$, and $\E_{n}$ indicates that the expectation is computed with respect to their corresponding posterior distributions given the previous computational evaluations, $f(x_1),\dots,f(x_n)$, and DM's responses, $a_1, \ldots, a_{m_n}$.

We call $\uei$ the \textit{expected improvement under utility uncertainty} and refer to the above policy as the $\uei$ policy. By construction, this sampling policy is one-step Bayes optimal.

\subsection{Computation and Maximization of $\uei$}
In contrast with the standard expected improvement, $\uei$ cannot be computed in closed form. However, as we show next, it can still be efficiently maximized. First, we introduce some notation. Making a slight abuse of notation, we denote $K_n(x,x)$ by $K_n(x)$. We also let $C_{n}(x)$ be the lower Cholesky factor of $K_n(x)$.

We note that, for any fixed $x\in\domain$, the time-$n$ posterior distribution of $f(x)$ is normal with mean $\mu_{n}(x)$ and covariance matrix $K_{n}(x)$. Therefore, we can express $f(x) = \mu_{n}(x) + C_{n}(x)Z$, where $Z$ is a $k$-variate standard normal random vector, and thus
\begin{equation*}
    \uei_n(x) = \expectation_{n}\left[\left\{ U(\mu_{n}(x) + C_{n}(x)Z;\theta) -  U_n^*(f;\theta)\right\}^{+} \right].
\end{equation*} 
This implies that we can compute $\uei_n(x)$ using Monte Carlo as summarized in Algorithm ~\ref{alg:ei}.
\begin{algorithm}[h]
\begin{algorithmic}[1]
\caption{Computation of $\uei$}
\label{alg:ei}
\REQUIRE{Point to be evaluated, $x$; number of Monte Carlo samples, $I$.}
\FOR{$i = 1,\ldots,I$}
\STATE{Draw samples $\theta^{(i)}$ and $Z^{(i)}$, and compute $\alpha^{(i)} := \left\{ U(\mu_{n}(x) + C_{n}(x)Z^{(i)};\theta^{(i)}) -  U_n^*(f;\theta^{(i)})\right\}^{+}$.}
\ENDFOR
\STATE{Estimate $\uei_n(x)$ by $\frac{1}{I}\sum_{i=1}^{I}\alpha^{(i)}$.}
\end{algorithmic}
\end{algorithm}

In principle, the above is enough to maximize $\uei$ using a derivative-free global optimization algorithm (for non-expensive functions). However, we could optimize $\uei$ more efficiently if we were able to leverage derivative information; this is the case using the derivative information we construct in the following proposition.
\begin{proposition}
\label{thm:uei_gradient}
Under mild regularity conditions, $\uei_n$ is differentiable almost everywhere, and its gradient, when it exists, is given by
\begin{equation*}\label{eqn:ei_gradient}
    \nabla\uei_n(x) = \expectation_{n}\left[\gamma_n(x, Z;\theta)\right],
\end{equation*}
where the expectation is over $\theta$ and $Z$, and
\begin{equation*}
    \gamma_n(x, Z; \theta) = \begin{cases}
0, \textnormal{ if } U(\mu_{n}(x) + C_{n}(x)Z;\theta) \leq  U_n^*(f;\theta)\\
\nabla U(\mu_{n}(x) + C_{n}(x)Z;\theta), \textnormal{ otherwise,}
\end{cases}
\end{equation*}
where the gradient $\nabla U(\mu_{n}(x) + C_{n}(x)Z;\theta)$ is with respect to $x$.
\end{proposition}
Thus, $\gamma$ provides an unbiased estimator of $\nabla \uei$, which can be used  within a gradient-based stochastic optimization algorithm, such as stochastic gradient ascent, to find stationary points of $\uei$.  We may then start stochastic gradient ascent from multiple starting points and 
use simulation to evaluate the $\uei$ for each and select the best.  By increasing the number of starting points, we may find a high-quality local optimum and asymptotically find a global optimum.

A formal statement and proof of Proposition ~\ref{eqn:ei_gradient} can be found in Appendix ~\ref{append:uei_gradient}.

\subsection{Computation of $\uei$  When $U$ Is Linear}
While the above approach can be used for efficiently maximizing $\uei$ for general utility functions, 
we can make maximization even more efficient for linear utility functions, the most widely used class in practice.
\begin{proposition}
\label{uei_linear}
Suppose that $\Theta\subset\R^k$ and $U(y;\theta) = \theta^\top y$ for all $\theta\in\Theta$ and $y\in\R^k$. Then,
\begin{equation*}
\uei_n(x) =  \expectation_n\left[\Delta_n(x;\theta)\Phi(\zeta) + \sigma_n(x;\theta)\varphi(\zeta)\right],
\end{equation*}
where the expectation is over $\theta$,
$\Delta_n(x;\theta) = \theta^\top \mu_n(x) - U_n^*(f;\theta)$,
$\sigma_n(x;\theta) = \sqrt{\theta^\top K_{n}(x)\theta}$,
$\zeta = \frac{\Delta_n(x;\theta)}{\sigma_n(x;\theta)}$,
and $\varphi$ and $\Phi$ are the standard normal density function and cumulative distribution function, respectively.
\end{proposition}
The result above shows that, when each $U(\cdot;\theta)$ is linear, the computation of $\uei$ essentially reduces to that of the standard expected improvement, modulo integrating the uncertainty over $\theta$. In particular, the uncertainty with respect to $Z$ can be integrated out. Moreover, in this case one can also derive an analogous result to Proposition ~\ref{eqn:ei_gradient} for computing the gradient of EI-UU in which the explicit dependence on $Z$ is eliminated as well. Formal statements and proofs of these two result can be found in Appendix ~\ref{append:uei_linear}.



\section{ADDITIONAL RELATED WORK}
\label{sec:lit}

The introduction discusses the two lines of most closely related work: 
the point-estimate approach pursued within preference-based reinforcement learning (PbRL); and multi-objective BayesOpt.

The most closely related work in PbRL is utility-based PbRL using trajectory utilities \citep{wirth2017survey}.  This variant of PbRL seeks to design a control policy to maximize the utility of a human subject using features computed from trajectories.
Work in this area includes \cite{akrour2014programming} and \cite{wirth2016model}.
Unlike our work, the uncertainty in utility function estimates is not considered when performing optimization. 

Multi-objective BayesOpt includes \cite{knowles2006, bautista2009,binois2015quantifying, shah2016pareto, feliot2017bayesian, hl16}. Multi-objective optimization cannot easily incorporate prior information about the DM's preferences, though several attempts have been made, mostly through modified Pareto-dominance criteria or weighted-sum approaches \citep{cvetkovic2002preferences, zitzler2004indicator, rachmawati2006preference}. Most of this work is outside the BayesOpt framework, with only three exceptions, which we describe below, known to us. 

\cite{feliot2018user} proposes a weighted version of the expected Pareto hypervolume improvement approach \citep{emmerich2006single} to focus the search on certain regions of the Pareto front. However, no method is provided for choosing weights from data, in contrast with our approach's ability to learn from interactions with the DM. Moreover, this method suffers the same computational limitations of the standard expected Pareto hypervolume improvement approach, limiting its applicability to at most three objectives \citep{hl16}. \cite{abdolshah2019multi} also proposes a weighted version of the expected Pareto hypervolume improvement approach to explore the region of the Pareto frontier satisfying a preference-order constraint over the objectives. Finally, \cite{paria2018flexible} proposes an approach based on random scalarizations. In contrast with our approach, no method is available for estimating the distribution of these scalarizations from data.


Another related literature is preferential BayesOpt.
Preferential BayesOpt \citep{gonzalez2017preferential} has been applied to realistic material design in computer graphics \citep{brochu2010tutorial} and  optimization of a parameterized control policy for robotic object handover in \citep{kupcsik2018learning}.
To apply preferential BayesOpt in our setting, we would choose pairs of treatments $x$ and $x'$, evaluate our computational model $f(x)$ and $f(x')$ for each, and obtain feedback from the DM on which treatment is preferred. Pairs of treatments would then be chosen to best support the goal of finding the DM's preferred design.
Critically, these methods do not 
use the attributes, $f(x)$, except to present them to the DM, but instead learn preferences directly as a function of $x$.
Thus, these methods tend to require many queries of the DM 
\citep{wirth2017efficient,pinsler2018sample}. Our approach leverages attribute observations to be more query efficient.



 Our work is also related to a line of research on adaptive utility elicitation
\citep{ chajewska1998utility,chajewska2000making,boutilier2002pomdp,boutilier2006constraint}.
Unlike in classical utility elicitation \citep{farquhar1984state, abbas2018foundations}, which has accurate estimation as its final goal, 
this work elicits the DM's utility function with the final goal of finding a good decision, even if this leaves residual uncertainty about the utility function \citep{braziunas2006computational}. However, this work assumes that attributes are inexpensive to evaluate, and that the space of designs is finite, preventing its use in our setting.

 Our work builds on BayesOpt \citep{ frazier2018tutorial}, a framework for optimization of time-consuming-to-evaluate black-box functions.
 Our proposed EI-UU acquisition function is a natural generalization of the classical \textit{expected improvement} acquisition function in standard BayesOpt \citep{movckus1975bayesian, jones1998efficient}. EI-UU also generalizes the \textit{expected improvement for composite functions} \citep{astudillo2019bayesian}, which can be obtained as a special case when $U$ is known.
 
 Our work is also related to \cite{frazier2011guessing}, which pursued a similar approach for the pure exploration multi-attribute multi-armed bandit problem with linear utility functions and without iterative interaction with the DM. Finally, an earlier version of this work, which considered linear utility functions only and no iterative interaction with the DM, appeared at \cite{astudillo2017multi}.

\section{EXPERIMENTS}
\label{sec:experiments}
We compare the performance of our sampling policies (EI-UU and TS-UU) against the policy that chooses the points to sample uniformly at random (Random),
and ParEGO \citep{knowles2006}, a popular for multi-objective BayesOpt algorithm. To understand the benefit obtained from preference information within our proposed sampling policies, we also report their performance without preference learning, i.e., where the distribution of $\theta$ remains equal to its prior distribution throughout all evaluations of $f$. In the plots, we distinguish from our sampling policies with preference learning by appending the subindex npl (which stands for ``no preference learning'').

In all problems, an initial stage of evaluations is performed using $2(d+1)$ points chosen uniformly at random over $\domain$.  A second stage (pictured in plots) is then performed using the given sampling policy.
For our algorithms, the outputs of $f$ are modeled using independent GP prior distributions. All GP models in our experiments have a constant
mean function and ARD Mate\'rn covariance
function with smoothness parameter equal to $5/2$; the associated hyperparameters are estimated
under a Bayesian approach. As proposed in \cite{snoek2012practical}, for all algorithms, except TS-UU, we use an averaged version of the acquisition function obtained by first drawing 10 samples of the GP hyperparameters, computing the acquisition function conditioned on each of these hyperparameters, and then averaging the results; for TS-UU, a single sample of the GP hyperparameters is used.

In all problems and for each replication, we draw one sample $\theta_{\textnormal{true}}$ from the prior distribution to obtain a true underlying utility function, $U(\cdot; \theta_{\textnormal{true}})$, which is used to obtain the preference information from the DM. The performance of the algorithms is reported with respect to this true underlying utility function. 

Our code and experiments are available at \url{https://github.com/RaulAstudillo06/BOPL}.

\subsection{Synthetic Test Functions}
The first three problems use well known test functions drawn from the  evolutionary multi-objective optimization literature \citep{van1999multiobjective,deb2005scalable, knowles2006}. We define these functions in detail in Appendix ~\ref{append:test_funcs}.

Results of these experiments are shown on a logarithmic scale in Figures ~\ref{fig:dtlz1a}, ~\ref{fig:dtlz2}, and  ~\ref{fig:vlmop3}. In these three test problems, EI-UU and TS-UU substantially outperform Random and ParEGO. In the first and third problems, EI-UU outperforms TS-UU, whereas in the second problem the opposite occurs. Throughout these problems, EI-UU greatly benefits from preference information. TS-UU also benefits from preference information, especially in the first two problems

\subsubsection{DTLZ1a With a Linear Utility}
\label{sec:dtlz1a}
A general form of this test function was first introduced in \cite{deb2005scalable}. The version we use was defined in \cite{knowles2006}. This function has $k=2$ attributes and is defined over $\domain = [0, 1]^6$. In this experiment, we use a linear utility function $U(y;\theta) = \theta y_1 + (1 - \theta)y_2$, and let the prior distribution on $\theta$ be uniform over $[0, 1]$.

\begin{figure}[h]
  \centering
  \includegraphics[width=0.49\textwidth]{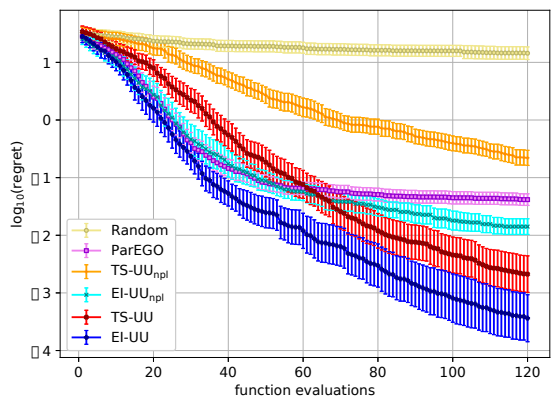}
  \caption{Average performance over 50 replications on the test problem described in \S\ref{sec:dtlz1a}. \label{fig:dtlz1a}}
\end{figure}
\subsubsection{DTLZ2 With a Quadratic Utility}
\label{sec:dtlz2}
This function was first introduced in a general form in \cite{deb2005scalable}. We use a concrete version of this function with $k=4$ attributes defined over  $\domain = [0, 1]^5$. Here, we use a quadratic utility function $U(y;\theta) = -\|y-\theta\|_2^2$, where $\theta$ is uniform over $\Theta$, and $\Theta$ consists of 8 points lying in the Pareto front of $f$, obtained as
\begin{equation*}
    \Theta = \left\{f(x): x_i\in\left\{\frac{i-1}{3}, \frac{i}{3}\right\}, i\leq3, \  x_4,x_5=0.5\right\}.
\end{equation*}
We envision that, in practice, such utility function could be used for finding designs with attributes as close as possible to an uncertain vector of ``ideal" attributes, which could take a range of values depending on the type of DM in question.

\begin{figure}[h]
  \centering
  \includegraphics[width=0.49\textwidth]{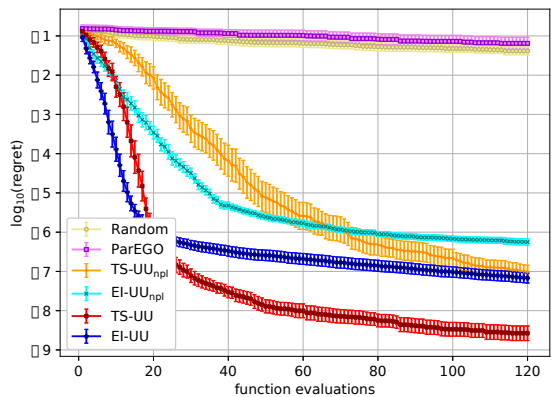}
  \caption{Average performance over 50 replications on the test problem described in \S\ref{sec:dtlz2}. \label{fig:dtlz2}}
\end{figure}

\subsubsection{VLMOP3 With an Exponential Utility}
\label{sec:vlmop3}
This test function first appeared in \cite{van1999multiobjective}. It has $k=3$ attributes and is defined over $\domain = [-3,3]^2$. Here, we use an exponential utility function
\begin{equation*}
        U(y;\theta) = \frac{1}{3}\sum_{j=1}^3\frac{1 - \exp(-\theta y_j)}{\theta},
    \end{equation*}
and let the prior on $\theta$ be uniform over $[0.1, 0.5]$. 

We note that, when $\theta \rightarrow 0$ from the right, the solution that maximizes $U(f(\cdot);\theta)$ converges to the solution that maximizes $\frac{1}{3}\sum_{j=1}^3f_j$ (neutral risk), whereas when $\theta \rightarrow \infty$, it converges to the one that maximizes $\min_{j=1,2,3}f_j$ (worst-case risk).  Therefore, if $f_1(x), \ldots, f_k(x)$ denote the outcome of an event under $k$ plausible scenarios with known likelihoods, $p_1,\ldots, p_k$ (in the above example ($p_j = 1/k$, $j=1,\ldots, k$),  this utility function provides a natural way to optimize the (expected) utility of a DM with uncertain (constant absolute) risk aversion with respect to this outcome.

\begin{figure}[h]
  \centering
  \includegraphics[width=0.49\textwidth]{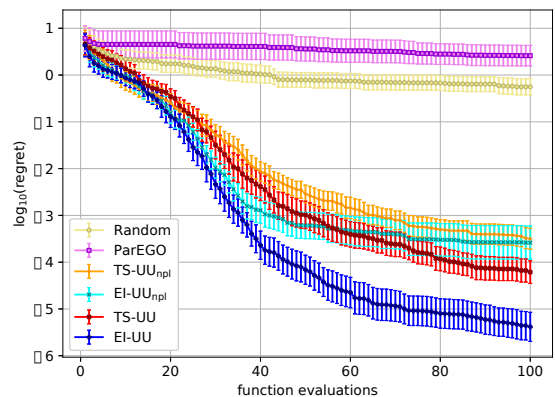}
  \caption{Average performance over 50 replications on the test problem described in \S\ref{sec:vlmop3}. \label{fig:vlmop3}}
\end{figure}

\subsection{Portfolio Simulation Optimization}
\label{sec:portfolio}
In this test problem, we use our algorithm to tune the hyperparameters of a trading strategy so as to maximize the return of a DM with an unknown risk aversion tolerance. We envision this as modeling a financial advisor that has many clients, each of which requires customized financial planning based on their own portfolio, and has a different risk tolerance.  Using choices made by past clients about which financial product they prefer, the financial advisor may form a probability distribution over utility functions  to use when using a computationally expensive simulation to develop a menu of options to show a new client.

We use CVXPortfolio \citep{cvxportfolio} to simulate and optimize the evolution of a portfolio over a period of four years, from Jan. 2012 through Dec.
2015, using open-source market data; the details of the simulation can be found in \S7.1 of \cite{cvxportfolio}. Here, $f$ has two outputs, the mean and (\textit{minus} the) standard deviation of the daily returns. We use a non-standard utility function 
that sets $U(y;\theta)$ to $y_1$ if $ \theta \leq y_2$ and $\infty$ otherwise.
This recovers the constrained optimization problem
that maximizes $f_1(x)$ subject to the constraint that $ \theta \leq f_2(x)$. Analogous to the case of linear utility functions, discussed in Proposition ~\ref{uei_linear}, it can be shown that for this class of utility functions, $\uei$ admits an expression similar to that of the \textit{constrained expected improvement} \citep{gardner14}.

Thus, in this setting we wish to maximize average return subject to an unknown constraint on the DM's risk tolerance level, $\theta$, which we assume is uniform over $[-2, -10]$ (recall that $f_2$ is \textit{minus} the standard deviation).
The hyperparameters to be tuned are the trade, hold, and risk aversion parameters over the domains $[0.1,1000]$, $[5.5, 8.]$, and $[0.1,100]$, respectively. Results are shown in Figure ~\ref{fig:portfolio}. Here, the optimal solution is unknown so we report the utility value instead. As before, EI-UU substantially outperforms Random and ParEGO, and is followed in performance by TS-UU. Both EI-UU and TS-UU benefit from preference information.

\begin{figure}[h]
  \centering
  \includegraphics[width=0.49\textwidth]{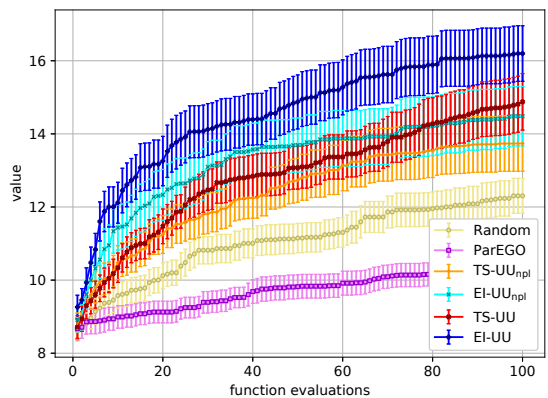}
  \caption{Average performance over 100 replications on the  test problem described in \S\ref{sec:portfolio}. \label{fig:portfolio}}
\end{figure}

\subsection{Optimization of Ambulance Bases}
\label{sec:ambulance}
In this test problem, we optimize the location of three ambulance bases according to the distribution of the response times. We consider $k=5$ attributes, representing the number of response times falling  within some given intervals of time, and assume a DM considers these attributes to choose the \textit{ideal} locations of the ambulance bases. We let $f_j$ be the number of response times falling  within $(5(j-1) \textnormal{ minutes},5j \textnormal{ minutes}]$, $j=1,\ldots, 4$, and $f_5$ be the number of those falling within the interval $(20 \textnormal{ minutes}, \infty)$. Due to the nature of these attributes, which are positive, we model their logarithms as GPs instead of the attributes directly. We then use the utility function
\begin{equation*}
    U(y;\theta) = \sum_{j=1}^5 \theta_j \frac{\exp{y_j}}{\sum_{i=1}^5\exp{y_i}},
\end{equation*}
which corresponds to a linear utility function over the  fraction of response times within the various intervals. Here, $\theta$ is taken to be uniform over the set $\Theta = \{\theta: \theta_1\geq\cdots\geq \theta_5 \geq 0\textnormal{ and } \sum_{j=1}^5\theta_j = 1\}$.

Results of this experiment are shown in Figure ~\ref{fig:ambulance}. As before, EI-UU substantially outperforms Random and ParEGO, and is followed by TS-UU. In contrast with all other test problems, however, here neither EI-UU nor TS-UU seem to benefit from preference information. A closer inspection to the data obtained from this experiment shows that there is highly concentrated region of designs that poses a high utility value for a wide range of values of $\theta$, which explains this behavior. This also suggests that our sampling policies are able to find \textit{robust} designs if they exist.

\begin{figure}[h]
  \centering
  \includegraphics[width=0.49\textwidth]{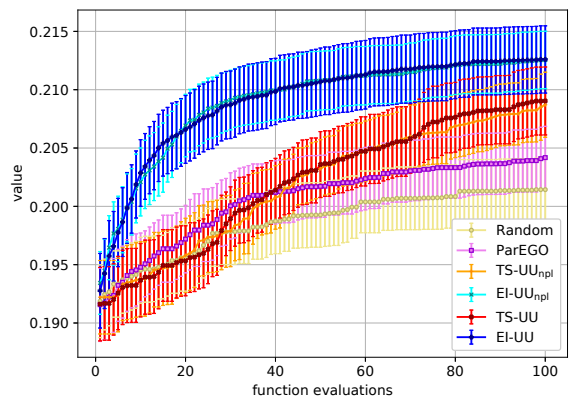}
  \caption{Average performance over 75 replications on the test problem described in \S\ref{sec:ambulance}. \label{fig:ambulance}}
\end{figure}

\section{CONCLUSION}
\label{sec:conclusion}
We introduced multi-attribute Bayesian optimization with preference learning, a novel approach for black-box global optimization of time-consuming-to-evaluate physical or computational experiments with multiple attributes that allows us to accommodate partial preference information in a principled way. By leveraging preference information, our approach is more efficient than multi-objective optimization approaches. By acknowledging uncertainty in the DM's preferences, our approach is more flexible and robust than single-objective optimization approaches that use a point estimate of the DM's utility function. Relevant directions for future work include developing more sophisticated policies for selecting the pairs of attributes to be shown to the DM, and using nonparametric models for estimating the DM's utility function.

\subsubsection*{Acknowledgements}
The authors were supported by 
NSF CMMI-1536895,
NSF CMMI-1254298,
AFOSR FA9550-15-1-0038, and
AFOSR FA9550-19-1-0283. The authors would like to thank anonymous reviewers for their comments.

\putbib
\end{bibunit}

\onecolumn
\setcounter{theorem}{0}
\appendix
\begin{bibunit}

\section{UNBIASED ESTIMATOR OF THE GRADIENT OF $\uei$}
\label{append:uei_gradient}
In this section we formally state and prove Proposition 1.
\begin{proposition}
	Suppose that $U(\cdot;\theta), \theta\in\Theta$ is differentiable for all $\theta\in\Theta$ and let $\domain'$ be an open subset of $\domain$ so that $\mu_n$ and $K_n$ are differentiable on $\domain'$ and there exists a measurable function $\eta:\R^k\rightarrow \R$ satisfying 
	\begin{enumerate}
		\item $\|\nabla U(\mu_{n}(x) + C_{n}(x)Z;\theta)\|<\eta(\theta, Z)$ for all $x\in\domain'$, $\theta\in\Theta$ and $Z\in\R^k$.
		\item $\expectation[\eta(\theta, Z)]<\infty$, where $Z$ is a $m$-variate standard normal random vector independent of $\theta$, and the expectation is over both $\theta$ and $Z$.
	\end{enumerate}
	Further, suppose that for almost every $\theta\in\Theta$ and $Z\in\R^k$ the set $\{x\in\domain' : U(\mu_{n}(x) + C_{n}(x)Z;\theta) =  U_n^*(f;\theta)\}$ is countable. Then,
	$\uei$ is differentiable on $\domain'$ and its gradient, when it exists, is given by
	\begin{equation*}\label{eqn:ei_gradient2}
	\nabla\uei(x) = \expectation\left[\gamma(x,\theta, Z)\right],
	\end{equation*}
	where the expectation is over $\theta$ and $Z$, and
	\begin{equation*}
	\gamma(x,\theta, Z) = \begin{cases}
	\nabla U(\mu_{n}(x) + C_{n}(x)Z;\theta), \textnormal{ if } U(\mu_{n}(x) + C_{n}(x)Z) >  U_n^*(f;\theta) ,\\
	0, \textnormal{ otherwise.}
	\end{cases}
	\end{equation*}
\end{proposition}
\begin{proof}
	From the given hypothesis it follows that, for any fixed $\theta\in\Theta$ and $Z\in\R^k$, the function $x \mapsto U(\mu_{n}(x) + C_{n}(x)Z;\theta)$ is differentiable on $\domain'$. This in turn  implies that the function $x \mapsto \{U(\mu_{n}(x) + C_{n}(x)Z;\theta -  U_n^*(f;\theta)\}_{+}$ is continuous on $\domain'$ and differentiable at every $x\in\domain'$ such that $U(\mu_{n}(x) + C_{n}(x)Z;\theta \neq  U_n^*(f;\theta)$, with gradient equal to $\gamma(x,\theta, Z)$. From our assumption that for almost every $\theta$ and $Z$ the set $\{x\in\domain : U(\mu_{n}(x) + C_{n}(x)Z;\theta) =  U_n^*(f;\theta)\}$ is countable, it follows that for almost every $\theta$ and $Z$ the function $x \mapsto \{U(\mu_{n}(x) + C_{n}(x)Z;\theta -  U_n^*(f;\theta)\}_{+}$ is continuous on $\domain'$ and differentiable on all $\domain'$, except maybe on a countable subset. Using this, along with conditions 1 and 2, and Theorem 1 in \citet{l1990unified}, the desired result follows.
\end{proof}
We note that, if one imposes the stronger condition $\expectation[\eta(\theta, Z)^2]<\infty$, then $\gamma$ has finite second moment, and thus this unbiased estimator of $\nabla\uei(x)$ can be used within stochastic gradient ascent to find a stationary point of $\uei$ \citep{bottou1998online}.

\section{COMPUTATION OF $\uei$ AND ITS GRADIENT WHEN $U$ IS LINEAR}
\label{append:uei_linear}
In this section we formally state and prove Propositions 2 and 3.

\begin{proposition}
	\label{uei_linear2}
	Suppose that $U(y;\theta) = \theta^\top y$ for all $\theta\in\Theta$ and $y\in\R^k$. Then,
	\begin{equation*}
	\uei(x) =  \expectation_n\left[\Delta_n(x;\theta)\Phi\left(\frac{\Delta_n(x;\theta)}{\sigma_n(x;\theta)}\right) + \sigma_n(x;\theta)\varphi\left(\frac{\Delta_n(x;\theta)}{\sigma_n(x;\theta)}\right)\right]
	\end{equation*}
	where the expectation is over $\theta$, $\Delta_n(x;\theta) = \theta^\top \mu_n(x) -  U_n^*(f;\theta)$, $\sigma_n(x;\theta) = \sqrt{\theta^\top K_{n}(x)\theta}$, and $\varphi$ and $\Phi$ are the standard normal probability density function and cumulative distribution function, respectively.
\end{proposition}
\begin{proof}
	Note that
	\begin{equation*}
	\uei(x) = \expectation_n\left[\expectation_{n}\left[\left\{ \theta^\top f(x) -  U_n^*(f;\theta)\right\}_{+} \mid\theta\right]\right].
	\end{equation*}
	Thus, it suffices to show that
	\begin{equation*}
	\expectation_{n}\left[\left\{ \theta^\top f(x) -  U_n^*(f;\theta)\right\}_{+} \mid\theta\right] = \Delta_n(x;\theta)\Phi\left(\frac{\Delta_n(x;\theta)}{\sigma_n(x;\theta)}\right) + \sigma_n(x;\theta)\varphi\left(\frac{\Delta_n(x;\theta)}{\sigma_n(x;\theta)}\right),
	\end{equation*}
	but this can be easily verified by noting that, conditioned on $\theta$, the time-$n$ posterior distribution of $\theta^\top f(x)$ is normal with mean $\theta^\top \mu_n(x)$ and variance $\theta^\top K_{n}(x)\theta$.
\end{proof}

\begin{proposition}
	\label{uei_gradient_linear}
	Suppose that $U(y;\theta) = \theta^\top y$ for all $\theta\in\Theta\subset \R^k$ and $y\in\R^k$, $\mu_n$ and $K_n$ are differentiable, and there exists a function $\eta:\Theta\rightarrow\R$ satisfying
	\begin{enumerate}
		\item $\left\|\left(\theta^\top \nabla\mu_n(x) \right)\Phi\left(\frac{\Delta_n(x;\theta)}{\sigma_n(x;\theta)}\right)  + \frac{\varphi\left(\frac{\Delta_n(x;\theta)}{\sigma_n(x;\theta)}\right)}{2\sigma_n(x;\theta)}\sum_{i,j=1}^m \theta_i\theta_j\nabla  K_{n}(x)_{i,j}\right\| \leq \eta(\theta)$ for all $x\in\domain$ and $\theta \in \Theta$.
		\item $\expectation[\eta(\theta)]<\infty$.
	\end{enumerate}
	Then, $\uei$ is differentiable and its gradient is given by
	\begin{equation*}
	\nabla\uei(x) =  \expectation_n\left[\left(\theta^\top \nabla\mu_n(x) \right)\Phi\left(\frac{\Delta_n(x;\theta)}{\sigma_n(x;\theta)}\right)  + \frac{\varphi\left(\frac{\Delta_n(x;\theta)}{\sigma_n(x;\theta)}\right)}{2\sigma_n(x;\theta)}\sum_{i,j=1}^m \theta_i\theta_j\nabla  K_{n}(x)_{i,j}\right].
	\end{equation*}
\end{proposition}
\begin{proof}
	Recall that
	\begin{equation*}
	\expectation_{n}\left[\left\{ \theta^\top f(x) -  U_n^*(f;\theta)\right\}_{+} \mid\theta\right] = \Delta_n(x;\theta)\Phi\left(\frac{\Delta_n(x;\theta)}{\sigma_n(x;\theta)}\right) + \sigma_n(x;\theta)\varphi\left(\frac{\Delta_n(x;\theta)}{\sigma_n(x;\theta)}\right).
	\end{equation*}
	Moreover, standard calculations show that
	\begin{equation*}
	\nabla\left[ \Delta_n(x;\theta)\Phi\left(\frac{\Delta_n(x;\theta)}{\sigma_n(x;\theta)}\right) \right] = \left(\theta^\top \nabla\mu_n(x) \right)\Phi\left(\frac{\Delta_n(x;\theta)}{\sigma_n(x;\theta)}\right) + \Delta_n(x;\theta)\varphi\left(\frac{\Delta_n(x;\theta)}{\sigma_n(x;\theta)}\right)\nabla\frac{\Delta_n(x;\theta)}{\sigma_n(x;\theta)},
	\end{equation*}
	and
	\begin{align*}
	\nabla\left[\sigma_n(x;\theta)\varphi\left(\frac{\Delta_n(x;\theta)}{\sigma_n(x;\theta)}\right)\right] &= \frac{\varphi\left(\frac{\Delta_n(x;\theta)}{\sigma_n(x;\theta)}\right)}{2\sigma_n(x;\theta)}\sum_{i,j=1}^m \theta_i\theta_j\nabla  K_{n}(x)_{i,j} + \sigma_n(x;\theta)\left[-\frac{\Delta_n(x;\theta)}{\sigma_n(x;\theta)}\varphi\left(\frac{\Delta_n(x;\theta)}{\sigma_n(x;\theta)}\right)\nabla\frac{\Delta_n(x;\theta)}{\sigma_n(x;\theta)}\right]\\
	&= \frac{\varphi\left(\frac{\Delta_n(x;\theta)}{\sigma_n(x;\theta)}\right)}{2\sigma_n(x;\theta)}\sum_{i,j=1}^m \theta_i\theta_j\nabla  K_{n}(x)_{i,j} - \Delta_n(x;\theta)\varphi\left(\frac{\Delta_n(x;\theta)}{\sigma_n(x;\theta)}\right)\nabla\frac{\Delta_n(x;\theta)}{\sigma_n(x;\theta)}.
	\end{align*}
	Thus, $\expectation_{n}\left[\left\{ \theta^\top f(x) -  U_n^*(f;\theta)\right\}_{+} \mid\theta\right]$ is a differentiable function of $x$, and its gradient is given by
	\begin{equation*}
	\nabla\expectation_{n}\left[\left\{ \theta^\top f(x) -  U_n^*(f;\theta)\right\}_{+} \mid\theta\right] = \left(\theta^\top \nabla\mu_n(x) \right)\Phi\left(\frac{\Delta_n(x;\theta)}{\sigma_n(x;\theta)}\right)  + \frac{\varphi\left(\frac{\Delta_n(x;\theta)}{\sigma_n(x;\theta)}\right)}{2\sigma_n(x;\theta)}\sum_{i,j=1}^m \theta_i\theta_j\nabla  K_{n}(x)_{i,j}.
	\end{equation*}
	From conditions 1 and 2, and theorem 16.8 in \cite{billingsley1995probability}, it follows that  $\uei$ is differentiable and its gradient is given by
	\begin{equation*}
	\nabla\uei(x) =  \expectation_n\left[\nabla\expectation_{n}\left[\left\{ \theta^\top f(x) -  U_n^*(f;\theta)\right\}_{+} \mid\theta\right]\right]
	\end{equation*}
	i.e.,
	\begin{equation*}
	\nabla\uei(x) =  \expectation_n\left[\left(\theta^\top \nabla\mu_n(x) \right)\Phi\left(\frac{\Delta_n(x;\theta)}{\sigma_n(x;\theta)}\right)  + \frac{\varphi\left(\frac{\Delta_n(x;\theta)}{\sigma_n(x;\theta)}\right)}{2\sigma_n(x;\theta)}\sum_{i,j=1}^m \theta_i\theta_j\nabla  K_{n}(x)_{i,j}\right].
	\end{equation*}
\end{proof}
We end by noting that if $\Theta$ is compact and $\mu_n$ and $K_n$ are both continuously differentiable, then
\begin{equation*}
(\theta, x) \rightarrow \left\|\left(\theta^\top \nabla\mu_n(x) \right)\Phi\left(\frac{\Delta_n(x;\theta)}{\sigma_n(x;\theta)}\right)  + \frac{\varphi\left(\frac{\Delta_n(x;\theta)}{\sigma_n(x;\theta)}\right)}{2\sigma_n(x;\theta)}\sum_{i,j=1}^m \theta_i\theta_j\nabla  K_{n}(x)_{i,j}\right\|
\end{equation*}
is continuous and thus attains its maximum value on $\Theta\times\domain$ (recall that $\domain$ is compact as well). Thus, in this case conditions 1 and 2 are satisfied by the constant function
\begin{equation*}
\eta \equiv \max_{(\theta, x)\in\Theta\times\domain} \left\|\left(\theta^\top \nabla\mu_n(x) \right)\Phi\left(\frac{\Delta_n(x;\theta)}{\sigma_n(x;\theta)}\right)  + \frac{\varphi\left(\frac{\Delta_n(x;\theta)}{\sigma_n(x;\theta)}\right)}{2\sigma_n(x;\theta)}\sum_{i,j=1}^m \theta_i\theta_j\nabla  K_{n}(x)_{i,j}\right\|.
\end{equation*}

\section{THOMPSON SAMPLING UNDER UTILITY UNCERTAINTY (TS-UU)}
\label{append:thompson}
Thompson sampling for utility uncertainty (TS-UU) generalizes the well-known Thompson sampling method \citep{thompson1933likelihood} to  our setting. TS-UU works as follows. It first samples $\theta$ from its posterior distribution.
Then, it samples $f$ from its Gaussian process posterior distribution.
The point at which it evaluates $f$ next is the one that maximizes $U(f(x);\theta)$ for the samples of $f$ and $\theta$. This contrasts with the point-estimate approach in that it samples $\theta$ from its posterior rather than simply setting it equal to a point estimate.  For example, if we implemented this point-estimate approach using standard Thompson sampling, we would sample only $f$ from its posterior and then optimize $U(f(x);\hat{\theta})$ where $\hat{\theta}$ is a point estimate, such as the maximum a posteriori estimate. TS-UU can induce substantially more exploration than this more classical approach.

TS-UU can be implemented by sampling $f(x)$ over a grid of points if $x$ is low-dimensional.  It can also be implemented for higher-dimensional $x$ by optimizing $f$ with a method for continuous nonlinear optimization (like CMA, \cite{hansen2016cma}), lazily sampling from the posterior on $f$ each new point that CMA wants to evaluate, conditioning on previous real and sampled evaluations. We use the latter approach in our numerical experiments.

\section{EXPLORATION AND EXPLOITATION TRADE-OFF}
One of the key properties of the classical expected improvement acquisition function is that it is increasing with respect to both the posterior mean and variance. This means that it prefers to sample points that are either promising with respect to our current knowledge or are still highly uncertain,  an appealing property for a sampling policy aiming to balance exploitation and exploration. The  following result shows that, under certain conditions, the $\uei$ sampling policy satisfies an analogous property.

\begin{proposition}
	Suppose that for every $\theta\in\Theta$  $U(\cdot;\theta)$ is convex and non-decreasing. Also suppose $x,x'\in\domain$ are such that $\mu_{n}(x)\geq\mu_{n}(x')$ and $K_{n}(x)\gtrsim K_{n}(x')$, where the first inequality is coordinate-wise and $\gtrsim$ denotes the partial order defined by the cone of positive semi-definite matrices. Then, 
	\begin{equation*}
	\uei_n(x)\geq\uei_n(x').
	\end{equation*}
\end{proposition}
\begin{proof}
	Since $K_{n}(x)\gtrsim K_{n}(x')$, we have that $f(x) \overset{d}{=} f(x') + (\mu_n(x) - \mu_n(x')) + W$, where $W$ is a $k$-variate normal random vector with zero mean and covariance matrix $K_{n}(x)- K_{n}(x')$ independent of $f(x')$. Thus,
	\begin{align*}
	\expectation_{n}\left[\left\{ U(f(x);\theta) -  U_n^*(f;\theta)\right\}_+ \mid\theta\right]
	&= \expectation_{n}\left[\left\{ U(f(x') + (\mu_n(x) - \mu_n(x')) + W;\theta) -  U_n^*(f;\theta)\right\}_+ \mid\theta\right]\\
	&\geq \expectation_{n}\left[\left\{ U(f(x')  + W;\theta) -  U_n^*(f;\theta)\right\}_+ \mid\theta\right]\\
	&= \expectation_{n}\left[\expectation_{n}\left[\left\{ U(f(x')  + W;\theta) -  U_n^*(f;\theta)\right\}_+ \mid\theta, f(x')\right] \right]\\
	&\geq \expectation_{n}\left[\left\{ U(f(x');\theta) -  U_n^*(f;\theta)\right\}_+ \mid \theta \right],
	\end{align*}
	where the first and second inequalities follow from the fact that the function $y\mapsto \left\{ U(y;\theta) -  U_n^*(f;\theta)\right\}_+$ is increasing and convex, respectively, along with Jensen's inequality. Finally, taking expectations with respect to $\theta$ yields the desired result.
\end{proof}
This result implies, for example, that for linear utility functions, the EI-UU sampling policy exhibits the behavior described above. We also note, however, that most utility functions used in practice are concave instead of convex.
\section{SYNTHETIC TEST FUNCTIONS DEFINITIONS}
\label{append:test_funcs}
\subsection{DTLZ1a}
A general form of this test function was first introduced in \cite{deb2005scalable}. The version we use was defined in \cite{knowles2006}.
It is defined over  $\domain = [0, 1]^6$, and has $k=2$ attributes given by
\begin{align*}
f_1(x) &= -0.5x_1\left((1 + g(x)\right)\\
f_2(x) &= -0.5(1 - x_1)\left((1 + g(x)\right),
\end{align*}
where
\begin{equation*}
g(x) = 100\left(5 + \sum_{i=2}^6\left[(x_i -0.5)^2 - \cos\left(2\pi(x_i - 0.5)\right)\right]\right).
\end{equation*}
The Pareto optimal set of designs consists of those such that $x_i = 0.5$, $i=2,\ldots, 6$, and $x_1$ may take any value in $[0,1]$. The Pareto front is a segment of the hyperplane $y_1 + y_2 = -0.5$.
\subsection{DTLZ2}
This function was first introduced in a general form in \cite{deb2005scalable}. In our experiment, we use a concrete version of it with $k=4$ attributes defined over $\domain = [0, 1]^5$. The attributes are  
\begin{align*}
f_1(x) &= -\left(1 + g(x)\right) \prod_{i=1}^3\cos\left(\frac{\pi}{2}x_i \right)\\
f_2(x) &= -(1 + g(x)) \left(\prod_{i=1}^2\cos\left(\frac{\pi}{2}x_i \right)\right)\sin\left(\frac{\pi}{2}x_3\right),\\
f_3(x) &= -(1 + g(x))\cos\left(\frac{\pi}{2}x_1\right)\sin\left(\frac{\pi}{2}x_2\right),\\
f_4(x) &= -(1 + g(x))\sin\left(\frac{\pi}{2}x_1\right),
\end{align*}
where
\begin{equation*}
g(x) =\sum_{i=4}^5(x_i - 0.5).
\end{equation*}.
\subsection{VLMOP3}
This test function first appeared in \cite{van1999multiobjective}. It is defined over $\domain = [-3, 3]^2$ and has $k=3$ attributes given by 
\begin{align*}
f_1(x) &= -0.5(x_1^2 + x_2^2) - \sin(x_1^2 + x_2^2),\\
f_2(x) &= -\frac{(3x_1 - 2x_2 + 4)^2}{8} - \frac{(x_1 - x_2 + 1)^2}{27} - 15,\\
f_3(x) &= -\frac{1}{x_1^2 + x_2^2 + 1} + 1.1\exp\left(-x_1^2 - x_2^2\right).
\end{align*}	
	
\putbib
\end{bibunit}

\end{document}